\documentclass[conference]{IEEEtran}
\pdfoutput=1 
\IEEEoverridecommandlockouts
\usepackage{lipsum}
\usepackage{cite}
\usepackage{bm}
\usepackage{amsmath,amssymb,amsfonts,amsthm}
\usepackage{environ}
\usepackage{algorithmic}
\usepackage{graphicx}
\usepackage{textcomp}
\usepackage{xcolor}
\usepackage{tabularx}
\usepackage{mathtools}
\usepackage{multirow}
\usepackage[makeroom]{cancel}
\def\BibTeX{{\rm B\kern-.05em{\sc i\kern-.025em b}\kern-.08em
		T\kern-.1667em\lower.7ex\hbox{E}\kern-.125emX}}
\usepackage{etoolbox}
\makeatletter
\patchcmd{\@makecaption}
{\scshape}
{}
{}
{}
\makeatletter
\patchcmd{\@makecaption}
{\\}
{.\ }
{}
{}
\makeatother
\def\tablename{Table}

\newtheorem{theorem}{Theorem}

\newtheorem{proposition}[theorem]{Proposition}

\DeclareMathOperator*{\argmax}{arg\,max}

\newcommand{\x}{\mathbf{x}}
\newcommand{\Ep}{\mathbb{E}_{p(\mathbf{x})}}
\newcommand{\A}{\mathbf{A}}
\newcommand{\I}{\mathbf{I}}
\newcommand{\W}{\mathbf{W}}

\newcommand{\D}{\mathbf{D}}
\newcommand{\E}{\mathbf{E}}
\newcommand{\F}{\mathbf{F}}
\newcommand{\diag}{{\rm diag}}
\newcommand{\s}{\mathbf{s}}
\newcommand{\bmu}{\bm{\mu}}
\newcommand{\bkappa}{\bm{\kappa}}
\newcommand{\btheta}{\bm{\theta}}
\newcommand{\bphi}{\bm{\phi}}

\newcommand{\bSigma}{\mathbf{\Sigma}}
\newcommand{\z}{\mathbf{z}}
\newcommand{\bb}{\mathbf{b}}
\newcommand{\y}{\mathbf{y}}
\newcommand{\N}{\mathcal{N}}
\newcommand{\C}{\mathbf{C}}
\newcommand{\U}{\mathbf{U}}
\newcommand{\V}{\mathbf{V}}
\newcommand{\B}{\mathbf{B}}
\newcommand{\J}{\mathbf{J}}

\begin{document}

\title{A Closer Look at Disentangling in $\beta$-VAE
	\thanks{*equal contribution, alphabetical order.}
}

\author{\IEEEauthorblockN{Harshvardhan Sikka$^*$\IEEEauthorrefmark{2}, Weishun Zhong$^*$\IEEEauthorrefmark{3}\IEEEauthorrefmark{6}, Jun Yin\IEEEauthorrefmark{4}, Cengiz Pehlevan\IEEEauthorrefmark{2}}\\

	\IEEEauthorblockA{\IEEEauthorrefmark{2}School of Engineering and Applied Sciences, Harvard University,
		Cambridge, MA, USA}
	\newline
	\IEEEauthorblockA{\IEEEauthorrefmark{3}Department of Physics, Massachusetts Institute of Technology, Cambridge, MA, USA}
	\newline
	\IEEEauthorblockA{\IEEEauthorrefmark{6}Center for Brain Sciences, Harvard University, Cambridge, MA, USA}
	\newline
	\IEEEauthorblockA{\IEEEauthorrefmark{4}Department of Physics, Harvard University, Cambridge, MA, USA}}

\maketitle

 \begin{abstract} In many data analysis tasks, it is beneficial to learn representations where each dimension is statistically independent and thus disentangled from the others. If data generating factors are also statistically independent, disentangled representations can be formed by Bayesian inference of latent variables. We examine a generalization of the Variational Autoencoder (VAE), $\beta$-VAE, for learning such representations using variational inference. $\beta$-VAE enforces conditional independence of its bottleneck neurons controlled by its hyperparameter $\beta$. This condition is in general not compatible with the statistical independence of latents. By providing analytical and numerical arguments, we show that this incompatibility leads to a non-monotonic inference performance in $\beta$-VAE with a finite optimal $\beta$.
\end{abstract}


\section{Introduction}

Although data sampled from the natural world appear to be high-dimensional, their variations can usually be explained using a much smaller number of latent factors. Both biological and artificial information processing systems exploit such structure and learn explicit representations that are faithful to data generative factors, known commonly as disentangled representations \cite{bengio2013representation}.  For example, sparse coding, an influential model of the primary visual cortex, proposes that the visual cortex neurons are coding for latent variables of natural scenes: oriented edges \cite{olshausen1997sparse}. A very popular method of extracting latent variables is by using the bottleneck neurons of deep autoencoders \cite{hinton1994autoencoders,alemi2018information}. 
In this paper, we examine unsupervised learning of disentangled representations  in the context of variational inference and a generalization of the Variational Autoencoder (VAE)  \cite{kingma2013auto}, $\beta$-VAE, developed specifically for disentangled representation learning \cite{higgins2017beta}.


We will adopt a probabilistic framework for latent-variable modeling of data \cite{kingma2019introduction}, where a generative model $p_{\btheta}(\x,\z)$ for data $\x$ and latent variables $\z$ is assumed:
\begin{align}
\label{eqn:generative}
p_{\btheta}(\x,\z) = \ p_{\btheta}(\x|\z)p(\z).
\end{align}
Here 
$\btheta$ denotes the parameters of our model, $p_{\btheta}(\x|\z)$ models the stochastic process that generates the data given the latent variables, and $p(\z)$ is the prior on the latent variables. An interpretable and common choice for $p(\z)$, and the subject of our paper, is a factorized distribution $p(\z) = \prod_{i=1}^k p_i(z_i)$, which implies statistical independence. Examples of models with independent priors include popular methods such as Independent Component Analysis \cite{hyvarinen2000independent,khemakhem2019variational} and Principal Component Analysis \cite{tipping1999probabilistic}.


While a common definition of learning disentangled representations has yet to be agreed upon \cite{bengio2013representation,kingma2013auto,locatello2018challenging,higgins2018towards}, extracting statistically independent latent factors is a natural choice \cite{bengio2013representation,hyvarinen2000independent} and is the definition we will adopt. Such a representation is efficient in that it carries no redundant information \cite{dayan2001theoretical}, and at the same time sufficient information to generate the data. 

In our probabilistic framework, the model posterior distribution $p_{\btheta}(\z|\x)$ allows inference of true latent variables. In principle, this could be used to form disentangled representations. However, model posterior is often intractable \cite{kingma2019introduction}, and variational methods are used to estimate it.

We focus on a state-of-the-art variational inference method for learning disentangled representations, $\beta$-VAE \cite{higgins2017beta}. The $\beta$-VAE training objective includes a hyperparameter, $\beta$, encapsulating the original VAE \cite{kingma2013auto} as a special case with choice $\beta = 1$. When $\beta$ is larger than unity, {\it conditional} independence of the learned representations at the bottleneck layer are enforced, corresponding to a conditional independence assumption on data generating latent variables, i.e. $p(\z|\x) = \prod_{i}p_i(z_i|\x)$ \cite{higgins2017beta}. However, as pointed above, a more natural assumption on latents is full statistical independence. Further, statistically independent latents are in general not conditionally independent. Given the popularity of VAEs in representation learning, it is important to understand the role of the $\beta$ hyperparameter in learning disentangled (statistically independent) latent variables.

Our main contributions are as follows:
\begin{enumerate}
	\item We provide general results about variational Bayesian inference in  $\beta$-VAE. Specifically, we prove that the $\beta$-VAE objective is non-increasing with increasing $\beta$, leading to worse reconstruction performance but more conditionally independent  representations.  Further, we argue that  latent variable inference performance generally tends to be non-monotonic in $\beta$.
	\item We introduce an analytically tractable model for $\beta$-VAE, specializing to statistically independent latent generative factors. 
	We analytically calculate the optimality conditions for this model, and numerically find that there is an optimal $\beta$ for the best inference of latent variables.
	\item We test our insights from the general theorems and the analytically tractable model using a realistic $\beta$-VAE architecture, using a synthetic MNIST dataset. Simulations agree well with our theory. 
\end{enumerate}

The rest of this paper is organized as follows. In Section \ref{bVAE}, we provide a review of variational inference and $\beta$-VAE. In Section \ref{general}, we  prove several theorems about variational inference in the context of $\beta$-VAE. In Section \ref{analytical}, we introduce our analytical results. In Section \ref{numerical}, we test our insights from the general theorems and the tractable models using a $\beta$-VAE architecture on a synthetic MNIST dataset. Finally, in Section VI we discuss our results and present our conclusions.

\section{Variational Inference and $\beta$-VAE}\label{bVAE}

Inference of latent variables in probabilistic models is often an intractable calculation \cite{kingma2013auto,kingma2019introduction}. 
%
%
Variational methods instead optimize over a set of tractable distributions, $q_{\bphi}(\z|\x)$, that best approximates $p_{\btheta}(\z|\x)$. We will refer to $q_{\bphi}(\z|\x)$ as the inference model. The difference between the two distributions can be quantified using the Kullback-Leibler (KL) divergence, which we call Model Inference Error (MIE):
\begin{align}
\label{eqn:model inference error}
{\rm MIE} \equiv \mathbb{E}_{p(\mathbf{x})}\left[ D_{KL}(q_{\bphi}({\bf z}|{\bf x}) || p_{\theta}({\bf z}|{\bf x} ) )\right] . 
\end{align}
We distinguish between MIE and the True Inference Error (TIE),
\begin{align}
\label{eqn:true inference error}
{\rm TIE} \equiv \mathbb{E}_{p(\mathbf{x})}\left[ D_{KL}(q_{\bphi}({\bf z}|{\bf x}) || p_{\text{g-t}}({\bf z}|{\bf x} ) )\right],
\end{align}
which can only be known when one has access to the underlying `ground-truth' data generative process and the ground-truth posterior, $p_{\text{g-t}}({\bf z}|{\bf x} )$.

VAEs fit the parameters of the probabilistic model and the variational distribution simultaneously.
A key identity in doing so is \cite{jordan1999introduction} 
\begin{align}
\label{eqn:main}
\ln p_{\btheta}({\bf x}) &- D_{KL}(q_{\bphi}(\z|\x) || p_{\btheta}(\z|\x ) ) \nonumber \\
&= \mathbb{E}_{q_{\bphi}(\z|\x)}\left[\ln p_{\btheta}(\x|\z)\right] -D_{KL}(q_{\bphi}(\z|\x) \| p(\z)).
\end{align}
Model fitting is done by maximizing the data log-likelihood, $\ln p_{\btheta}(\x)$, under model parameters. Because the KL divergence is non-negative, the right hand side of \eqref{eqn:main} serves as a lower bound for $\ln p_{\btheta}(\x)$ and is called the Evidence Lower Bound (ELBO)
\begin{align}\label{cost}
&{\rm ELBO}(\btheta,\bphi) \nonumber \\ &\qquad \equiv \mathbb{E}_{q_{\bphi}(\z | \x)}\left[\log p_{\btheta}(\x | \z)\right]- D_{KL}\left(q_{\bphi}(\z|\x) \| p(\z)\right).
\end{align}
%
%
%
VAE parameterizes the distributions $p_{\btheta}(\x|\z)$ and $q_{\bphi}(\z|\x)$ with neural networks, and maximizes ELBO as a proxy for maximizing the data likelihood.

The neural network realization of the $p_{\btheta}(\x|\z)$ is referred to as a decoder \cite{kingma2013auto}. Once the VAE is trained, the decoder can be used as to generate new samples from the model data distribution \cite{kingma2013auto,doersch2016tutorial}. The term $\mathbb{E}_{q_{\bphi}(\z | \x)}\left[\log p_{\btheta}(\x | \z)\right]$ measures the reconstruction performance of the generative model. We will refer to it as the reconstruction objective.

The neural network realization of the inference model is referred to as an encoder \cite{kingma2013auto}. Its outputs constitute a bottleneck layer and  represent inferred latent variables. Note that the MIE calculated from this representation appears on the left hand side of \eqref{eqn:main}. 

$\beta$-VAE is an extension of the traditional VAE, where an extra, adjustable hyperparameter $\beta$ is placed in the training objective: 
\begin{align}\label{bVAEobj}
\mathcal{L}(\btheta,\bphi;\beta) = \mathbb{E}_{q_{\bphi}(\z | \x)}\left[\log p_{\btheta}(\x | \z)\right]-\beta D_{K L}\left(q_{\bphi}(\z | \x) \| p(\z)\right).
\end{align}
Specifically, when  $\beta = 1$, the  $\beta$-VAE is equivalent to VAE and ${\rm ELBO}(\btheta,\bphi) = \mathcal{L}(\btheta,\bphi;1)$. 

Higher values of $\beta$ emphasizes the KL divergence between the inference model $q_{\bphi}(\z|\x)$ and the independent prior $p(\z)$ in the objective \eqref{bVAEobj}. Smaller values of the KL divergence favor a conditionally independent inference model. This can be used to learn disentangled representations of conditionally independent latent variables, whose probability distributions factorize when conditioned on data \cite{higgins2017beta,burgess2018understanding}.

However, as alluded to in our introduction, in many cases of interest and application \cite{Huang_2018_ECCV,lample2017fader,karras2019style}, latent variables are conditionally dependent while being independent \cite{hyvarinen2000independent},\cite{tipping1999probabilistic}. We will encounter an analytically tractable case in Section \ref{analytical}. In such cases, it is not clear if a $\beta$ different than 1 helps learning a disentangled representation which extracts statistically independent latent factors. Our goal in the remaining of this paper is to examine this case analytically and numerically.

For convenience, we also attach a table of terms and corresponding mathematical expressions used throughout the paper (Table \ref{table:1}).

\begin{table}[h!]
	\centering
	\scalebox{1.02}{
		\def\arraystretch{1.85}
		\begin{tabular}{ |c|c|c| } 
			\hline
			\textbf{Term} & \textbf{Mathematical Expression} \\
			\hline
			Prior & $p(\z)$ \\ 
			\hline
			Model Posterior & $p_{\btheta}(\z|\x)$ \\ 
			\hline
			Ground-Truth Posterior & $p_{\text{g-t}}(\z|\x)$ \\ 
			\hline
			Inference Model & $q_{\bphi}(\z|\x)$ \\ 
			\hline
			Data Log-Likelihood & $\log p_{\btheta}(\x)$ \\ 
			\hline
			Reconstruction Objective & $\mathbb{E}_{q_{\bphi}(\z | \x)}\left[\log p_{\btheta}(\x | \z)\right]$ \\ 
			\hline
			\shortstack{Conditional Independence Loss} & $D_{K L}\left(q_{\bphi}(\z | \x) \| p(\z)\right)$  \\ 
			\hline
			$\rm{MIE}$ & $\mathbb{E}_{p(\mathbf{x})}[ D_{KL}(q_{\bphi}({\bf z}|{\bf x}) || p_{\theta}({\bf z}|{\bf x} ) )]$  \\ 
			\hline
			$\rm{TIE}$ & $\mathbb{E}_{p(\mathbf{x})}[ D_{KL}(q_{\bphi}({\bf z}|{\bf x}) || p_{\text{g-t}}({\bf z}|{\bf x} ) )]$  \\ 
			\hline
			\shortstack{Evidence Lower Bound \\ (ELBO)}  & \shortstack{$\mathbb{E}_{q_{\bphi}(\z | \x)}\left[\log p_{\btheta}(\x | \z)\right]$ \\ $- D_{K L}\left(q_{\bphi}(\z | \x) \| p(\z)\right)$ }\\ 
			\hline
	\end{tabular}}
	\vskip 0.1in
	\caption[Table caption text]{Table of terms and corresponding mathematical expressions.}
	\label{table:1}
\end{table}

\section{How $\beta$ Affects Model Performance and Inference of Latent Variables}\label{general}

In this section, we provide general statements on the effect of the $\beta$ parameter on the representation learning and the generative functions of  $\beta$-VAE. We do this by proving propositions about how various terms in the identity \eqref{eqn:main} change as a function of $\beta$. Our first two propositions imply that increasing $\beta$ worsens the quality of reconstructed samples while improving conditional disentangling. While these points have been shown in simulations \cite{higgins2017beta,burgess2018understanding}, here we provide analytical statements. Our last proposition gives a handle on understanding behavior of MIE through ELBO.


In the following, we will denote optimal parameters of a $\beta$-VAE that maximizes the objective \eqref{bVAEobj} by $\btheta^*$ and $\bphi^*$. They are given as a solution to 
\begin{align}\label{opt}
\frac{\partial \mathcal{L}}{\partial \btheta} = \bm{0}, \qquad \frac{\partial \mathcal{L}}{\partial \bphi} = \bm{0}.
\end{align}
We denote the value of the optimal objective by
\begin{align}
\mathcal{L}^*(\beta) \equiv \mathcal{L}(\btheta^*(\beta),\bphi^*(\beta),\beta),
\end{align}
and the value of ELBO at the optimal point by
\begin{align}
{\rm ELBO}^*(\beta) \equiv {\rm ELBO}  (\btheta^*(\beta),\bphi^*(\beta)).  
\end{align}

Our first proposition concerns the behavior of $\mathcal{L}^*(\beta)$ as a function of $\beta$.

\begin{proposition}\label{prop1}
	The optimal value of the $\beta$-VAE objective, $\mathcal{L^*}(\beta)$, is non-increasing with increasing $\beta$:
	\begin{align}
	\frac{\partial \mathcal{L^*}(\beta) }{\partial \beta} = - D_{K L}\left(q_{\bphi^*}(\mathbf{z} | \mathbf{x}) \| p(\mathbf{z})\right) \leq 0.
	\end{align}
\end{proposition}
\begin{proof}
	Follows from an application of the chain rule, the optimality conditions \eqref{opt}, and the nonegativity of the KL-divergence:
	\begin{align}
	\frac{\partial \mathcal{L^*} }{\partial \beta} &= \left.\left(\frac{\partial \mathcal{L} }{\partial \btheta}\cdot \frac{\partial \btheta }{\partial \beta} + \frac{\partial \mathcal{L} }{\partial \bphi}\cdot \frac{\partial \bphi }{\partial \beta} + \frac{\partial \mathcal{L} }{\partial \beta}\right)\right|_{\btheta =\btheta^*, \bphi =\bphi^* } \nonumber \\
	&= -D_{KL}(q_{\bphi^*}({\bf z}|{\bf x}) || p({\bf z})) \leq 0.
	\end{align}
\end{proof}

The next proposition shows how the two terms in $\mathcal{L}^*$ change with $\beta$. 

\begin{proposition}\label{prop2}  The KL divergence between the inference model and the prior is non-increasing with increasing $\beta$:
	\begin{align}\label{eq:qp}
	\frac{d}{d\beta} D_{KL}(q_{\phi^*}({\bf z}|{\bf x}) || p({\bf z})) \leq 0.
	\end{align}
	Together with Proposition \eqref{prop1}, this implies that
	\begin{align}\label{eq:re}
	\frac{d \, \mathbb{E}_{q_{\bphi^*}(\z | \x)}\left[\log p_{\btheta^*}(\x | \z)\right]}{d\beta} \leq 0
	\end{align}
\end{proposition}
\begin{proof}
	See Appendix \ref{pprop2}.
\end{proof}

The next proposition is about the behavior of ${\rm ELBO}^*$.

\begin{proposition}\label{prop3}
	${\rm ELBO}^*$ is maximized at $\beta = 1$. 
	%
\end{proposition}
\begin{proof} Note that by definition 
	\begin{align}\label{lelbo}
	\mathcal{L} = {\rm ELBO} + (1-\beta)D_{KL}(q_{\bphi}(\z|\x) || p({\bf z})).
	\end{align}
	By evaluating \eqref{lelbo} at $\btheta=\btheta^*$ and $\bphi=\bphi^*$, and the chain rule, we get:
	\begin{align}
	\frac{d\,{\rm ELBO^*(\beta)}}{d\beta}&= \frac{d}{d\beta} \left[\mathcal{L}^* -(1-\beta)D_{KL}(q_{\bphi^*}(\z|\x) || p({\bf z}))\right] \nonumber \\
	&= (\beta-1) \frac{d}{d\beta}D_{KL}(q_{\bphi^*}({\bf z}|{\bf x}) || p({\bf z})).
	\end{align}
	The proposition follows from this result and \eqref{eq:qp}. 
\end{proof}

For simplicity of notation, we presented most of our formulas and propositions for a single data point. All our results generalize to the case where one averages over the data distribution $p(\x)$, or a finite training set. 

Inference of latent variables, measured by MIE, is affected by $\beta$ as well. In the $\beta \to \infty$ limit the inference model becomes more and more conditionally independent, deviating from the model posterior. Is the behavior monotonic? While MIE is not explicitly calculable, we can get a hint of its behavior by rearranging \eqref{eqn:main}, and evaluating it at the optimal $\beta$-VAE parameters:
\begin{align}\label{MIEELBO}
{\rm MIE}(\beta) = \Ep\left[\ln p_{\btheta^*(\beta)}(\x)-{\rm ELBO}^*(\beta)\right].
\end{align}
As reconstruction performance worsens with $\beta$, it is reasonable to expect that the data likelihood decreases with $\beta$. Because ELBO is non-monotonic with a maximum, even if the data log-likehood was monotonic with $\beta$, we can expect a non-monotonic behavior of MIE with an optimal value.  In the next section, we will see two specific examples of this.

\section{Analytical Results}\label{analytical}

In this section we demonstrate our general theory for two different analytically tractable cases. 

\subsection{$\beta$-VAE with a fixed decoder does not lead to better disentangling}
A simple case is when the decoder of the $\beta$-VAE is not trained. In our notation, this amounts to $\btheta$ being fixed. Then the $\beta$-VAE objective \eqref{bVAEobj} only trains the encoder network and the inference model, $q_{\bphi}(\z\|\x)$. We can deduce the behavior of MIE as a function of $\beta$ from \eqref{MIEELBO}. The data likelihood, $p_{\btheta}(\x)$, does not change as a function of training. ${\rm ELBO}^*$ is maximized at $\beta = 1$ from Proposition \ref{prop3}, which can be seen to apply to fixed $\btheta$. This means MIE is minimum at $\beta = 1$. In this case, $\beta=1$, or the original VAE is best at learning the true latent variables.

\subsection{Optimal $\beta$ values in an analytically tractable model}

Next, we present a tractable VAE model, in which we can explicitly calculate the $\beta$-dependence in every term in eq. \eqref{eqn:main}. 

We assume that our data $\x$ comes from mixing of ground truth latent variables (or sources) $\s \in \mathbb{R}^{k}$ through a mixing matrix $\A \in \mathbb{R}^{N \times k}$, then corrupted by noise $\bm{\eta} \in \mathbb{R}^{N}$, 
\begin{equation}
\label{source}
\mathbf{x} = \mathbf{A} \mathbf{s} + \bm{\eta}.
\end{equation}
We assume $\s \sim \N(\bm{0},\I_k)$, $ \bm{\eta} \sim \N(\bm{0},\I_N)$. The data distribution is found to be,
\begin{equation}
\label{eqn:data}
p(\x) = \N(\bm{0},\A\A^{\top}+\I_N) \equiv \N(\bm{0},\bSigma_{\x}).
\end{equation}
We denote a $d \times d$ identity matrix as $\I_d$. In this model we can calculate the ground-truth posterior exactly (see Appendix \ref{app:gtposterior} for details): 
\begin{align}\label{eqn:true posterior}
p_{\text{g-t}}(\s |\x) &= \mathcal{N}(\bmu_{\s|\x}, \bSigma_{\s|\x}), \nonumber \\
\text{with}\;\;  \bmu_{\s|\x} &= (\A^{\top}\A + \mathbf{I}_k)^{-1}\A^{\top}\x \nonumber \\
\text{and}\;\;  \bSigma_{\s|\x} &= (\A^{\top}\A + \mathbf{I}_k)^{-1}.
\end{align}
Note that the covariance matrix of the posterior is non-diagonal. Even though the latent factors are statistically independent, when conditioned on data they are dependent. Therefore, we expect a non-trivial dependence of MIE and TIE on the hyperparameter $\beta$.


Our encoder $q_{\bphi}(\z|\x)$ contains a fully-connected layer $\{\W^{\mu},\bb^{\mu}\}$ with linear activation that codes for the mean $\bmu_{\z}$ of the latent variables $\z$, and a fully-connected layer $\{\W^{\sigma},\bb^{\sigma}\}$ with exponential activation that codes for the diagonal part of the covariance matrix $\bSigma_{\z}$. Given an input $\x \in \mathbb{R}^N$, we generate latent variables $\z\sim \N(\bmu_{\z},\bSigma_{\z}) \in \mathbb{R}^k$ by
\begin{align}
\label{eqn:musigma}
\bmu_{\z} = \W^{\mu} \x + \bb^\mu, \quad \bSigma_{\z} = \diag ( \exp(\W^{\sigma} \x + \bb^\sigma)),
\end{align}
where the $\diag$ operation maps vectors in $\mathbb{R}^{k}$ to the diagonal of a diagonals matrix in $\mathbb{R}^{k\times k}$. The exponential nonlinearity in the definition of the covariance matrix acts elementwise and prevents negative covariances.

Our decoder consists of a single fully-connected layer $\{\D,\bb^{D}\}$ with linear activations. We assume the output $\y \in \mathbb{R}^N$ is normally distributed, $\y \sim \N (\D \z+\bb^D,\sigma_y^2 \I_N)$, where $\sigma_y^2$ is a hyperparameter. Without loss of generality, from now on we choose $\sigma^2_y =1$. 

The decoder defines  $p_{\btheta}(\x | \z)$. The full data likelihood can be calculated using the prior $p(\z) = \mathcal{N}(\bm{0},\I_k)$ through $p_{\btheta}(\x) = \int d\z\, p_{\btheta}(\x | \z)p(\z)$. With this setup, our decoder is fully capable of modeling the data generative process \eqref{eqn:data}, by choosing $\D = \A$, $\bb^D = \bm{0}$ and $\sigma^2_y = 1$. Any deviation from these parameters will be due to the encoder, or the inference model, deviating from the ground-truth distribution.

In order to solve this model, we integrate out data (i.e., performing $\mathbb{E}_{\x\sim p(\x)}[..]$, using eq. \eqref{eqn:data}) in the $\beta$-VAE objective in eq. \eqref{cost} to arrive at (see Appendix \ref{app:intoutdata} for details)
\begin{align}
\label{eqn:objective_nox}
\mathcal{L}_{\beta} = &-\frac{1}{2} \bigg\{ \text{Tr} \bigg[ (\D \W^{\mu} - \mathbf{I}_N) \bSigma_{\x} (\D \W^{\mu} - \mathbf{I}_N)^{\top} \bigg] \nonumber \\
&+ \beta \text{Tr} \bigg[ \W^{\mu} \bSigma_{\x} (\W^{\mu})^{\top} \bigg] \nonumber \\
&+ \sum_i^k \left(\left[\D^{\top}\D\right]_{ii} + \beta\right) e^ {\frac{1}{2}\left[\W^{\sigma}\bSigma_{\x} (\W^{\sigma})^{\top}\right]_{ii} + b_i^{\sigma}} \nonumber \\
&+ (\D \bb^{\mu} + \bb^D)^2 + \beta (\bb^{\mu})^2 - \beta \sum_i^k b^{\sigma}_i  \bigg\}.
\end{align}
We optimize over the network parameters, which amounts to setting the partial derivative of $\mathcal{L}_{\beta}$ with respect to $\{\W^{\mu},\bb^{\mu}, \W^{\sigma},\bb^{\sigma}, \D,\bb^{D}\}$ to zero. Upon simplifying, we find (see Appendix \ref{app:derivatives} for details)
\begin{equation}
\label{eqn:bmubD}
\bb^{\mu} = \bb^{D}=0, 
\end{equation}{}
and the remaining equations are ($a=1,...,N;\; b=1,...,k$):
\begin{align}\label{eqn:equations}
0 &= \left[\left( \D^{\top} (\D \W^{\mu} - \mathbf{I}_N) + \beta \W^{\mu} \right)\bSigma_{\x} \right]_{ab}, \nonumber \\
0 &= \left[ (\D \W^{\mu} - \I_N) \bSigma_{\x} (\W^{\mu})^{\top} \right]_{ab} + D_{ab} e^{ \frac{1}{2}\left[\W^{\sigma}\bSigma_{\x} (\W^{\sigma})^{\top}\right]_{bb} + b^{\sigma}_b }, \nonumber \\
\hspace{-1cm} 0 &= \left(\left[\D^{\top}\D\right]_{aa} + \beta \right) e^{ \frac{1}{2}\left[\W^{\sigma}\bSigma_{\x}\right]_{ab}W^{\sigma}_{ab} + b^{\sigma}_a} \left[\W^{\sigma}\bSigma_{\x}\right]_{ab}, \nonumber \\
0 &= \left( \left[\D^{\top}\D\right]_{aa} + \beta \right) e^{  \frac{1}{2}\left[\W^{\sigma}\bSigma_{\x} (\W^{\sigma})^{\top}\right]_{aa} + b^{\sigma}_a } - \beta.
\end{align}
We can calculate the model posterior distribution $p_{\theta}(\z |\mathbf{x})$ at the network optimum, eqs. \eqref{eqn:bmubD} and \eqref{eqn:equations}. Using Bayes' rule we find (see Appendix \ref{app:modelposterior})
\begin{align}\label{eqn:model posterior}
p_{\theta}(\z |\x) &= \mathcal{N}(\bmu_{\z|\x}, \bSigma_{\z|\x}), \nonumber \\
\text{with}\;\;  \bmu_{\z|\x} &= (\D^{\top}\D + \I_k)^{-1}\D^{\top}\x  \nonumber \\
\text{and}\;\;  \bSigma_{\z|\x} &= (\D^{\top}\D + \I_k)^{-1}.
\end{align}
Note that when $\A = \D$, eq. \eqref{eqn:model posterior} reduces to  eq. \eqref{eqn:true posterior}, and the model posterior matches with the ground-truth posterior.
We are interested in the inference errors MIE and TIE, eqs. \eqref{eqn:model inference error} and \eqref{eqn:true inference error}.
Upon integrating out the data, we find (see Appendix \ref{app:modelposterior} for derivations)
\begin{align}
\label{eqn:mie/tie}
&{\rm MIE/TIE} \nonumber \\
= &\frac{1}{2} \bigg\{ \sum_i^k E_{ii}^{-1}\exp\bigg[ \frac{1}{2}\bigg(\W^{\sigma}\bSigma_{\x}(\W^{\sigma})^{\top} \bigg)_{ii} + b^{\sigma}_i \bigg] - \sum_i^k b^{\sigma}_i \nonumber \\
& + \text{Tr}\log \mathbf{E} + \text{Tr}\bigg[(\mathbf{F} - \W^{\mu})^{\top}\mathbf{E}^{-1}(\mathbf{F} - \W^{\mu})\bSigma_{\x} \bigg] -k \bigg\},
\end{align}
where for MIE
\begin{align}
\E = (\D^{\top}\D+\I_k)^{-1}, \quad
\F = (\D^{\top}\D+\I_k)^{-1}\D^{\top},
\end{align}
and for TIE
\begin{align}
\E = (\A^{\top}\A + \mathbf{I}_k)^{-1},\quad
\F = (\A^{\top}\A + \mathbf{I}_k)^{-1}\A^{\top}.
\end{align}
%

%
\begin{figure}[t]
	\centering
	\includegraphics[width=9cm]{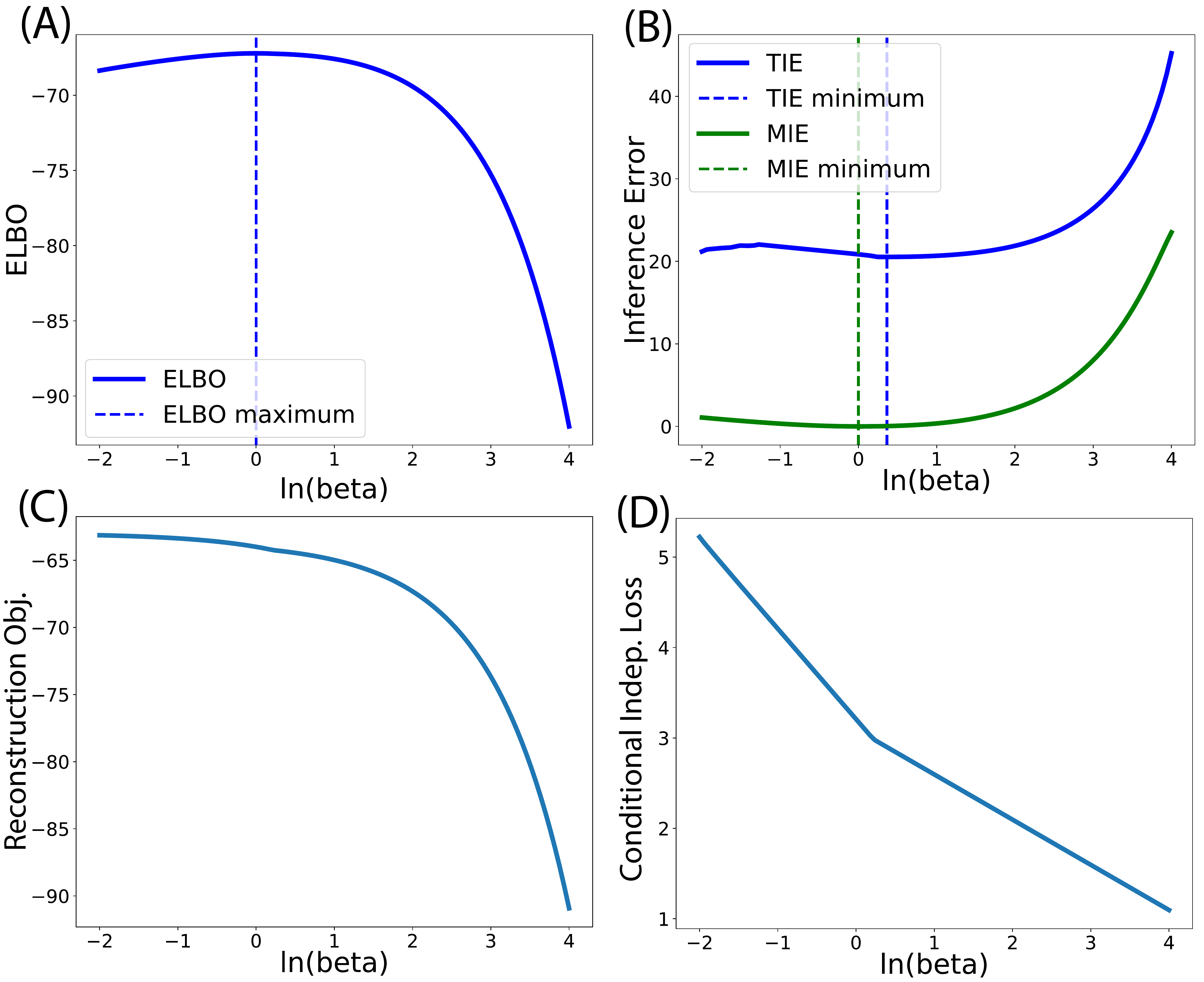}
	\caption{$\beta$-dependence of various quantities at the optimal parameter configuration of $\beta$-VAE. (A) ELBO as a function of $\beta$. (B) MIE/TIE as a function of $\beta$. (C) Reconstruction objective as a function of $\beta$. (D) Conditional Independence Loss as a function of $\beta$. In these plots, we averaged the plotted quantities over the data distribution.}
	\label{fig:panel}
\end{figure}

As an example, we numerically solve eq. \eqref{eqn:equations} for $N=128, k=2$, $A_{ij}=1/2(1+\delta_{ij})$, and use the optimal network parameters $\{\W^{\mu*},\bb^{\mu*}, \W^{\sigma*},\bb^{\sigma*}, \D^*,\bb^{D*}\}$ to calculate ELBO (Fig. \ref{fig:panel}(A)) and inference errors (Fig. \ref{fig:panel}(B)). We  see that ELBO is maximized at $\beta = 1$, while the inference error is not monotonically decreasing and has a minimum at some $\beta$. This confirms the theory we outlined earlier. Also, data log-likelihood is monotonically decreasing with $\beta$ (not shown). We further calculate individual terms in the ELBO: the reconstruction objective (Fig. \ref{fig:panel}(C)), $\mathbb{E}_{q_{\bphi}(\z | \x)}\left[\log p_{\btheta}(\x | \z)\right]$, and the conditional Independence Loss (Fig. \ref{fig:panel}(D)), $D_{KL}\left(q_{\bphi}(\z|\x) \| p(\z)\right)$. Indeed both terms are monotonically decreasing with $\beta$, confirming our propositions.

\section{Numerical Simulations}\label{numerical}

In this section, we examine a deep, nonlinear $\beta$-VAE on a synthetic dataset. The dataset is generated according to eq. \eqref{source} by mixing 10 MNIST digits, arranged as columns of a matrix $\A$, with ground truth sources, $\s \sim \N(\bm{0},\I_k)$, and subsequently adding a noise $\bm{\eta} \sim \N(\bm{0},\I_N)$. Other experimental setups and corresponding datasets that were explored are included in Appendix \ref{app_sims} (Fig. \ref{fig:panel3}).


The encoder, $q_{\bphi}(\z|\x)$, consists of three feed-forward fully-connected layers with tanh activations, ending in two separate output layers encoding the mean of the latent variables $\z$, $\bmu_{\z}$, and the variance, $\bSigma_{\z}$. These are each parameterized by $k$ encoding units. The decoder, $p_{\btheta}(\x | \z)$, consists of three feed-forward fully-connected layers with tanh activation functions, which takes its input from the encoder, and outputs the reconstructed image. Model details are included in Appendix C.  


After training, we  calculate individual terms in the $\beta$-VAE objective and demonstrate their dependence on  $\beta$. These terms correspond to the Reconstruction Objective, (Fig. \ref{fig:panel2}(C)), and the conditional Independence Loss, (Fig. \ref{fig:panel2}(D)). As we observed in the analytically tractable case, and predicted by our theory, these terms are decreasing with $\beta$. Correspondingly, after being maximized around $\beta = 1$ the entire ELBO term decreases with $\beta$ (Fig. \ref{fig:panel2}(A)). We also calculate the TIE for the $\beta$-VAE at various $\beta$, which follows a non-monotonic trend and has an optimal $\beta$ (Fig. \ref{fig:panel2}(B)).

\begin{figure}
	\centering
	\includegraphics[width=9cm]{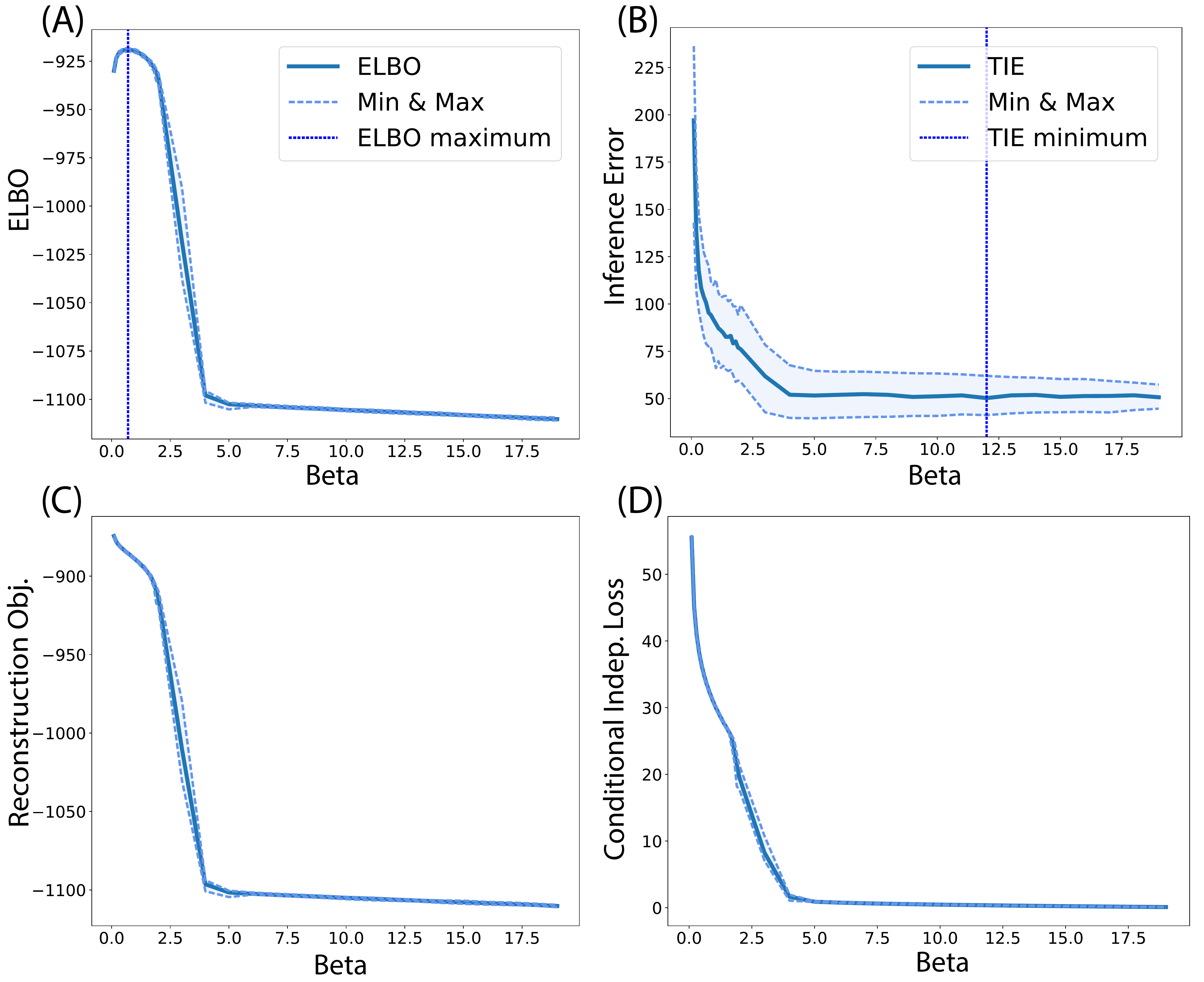}
	\caption{Values for error terms across 100 random initializations of the network. Solid line represents the average. Dashed lines around the solid line represent the minimum and maximum values, and vertical dashed line represent the extremum. (A) ELBO as a function of $\beta$. (B) TIE as a function of $\beta$. (C) Reconstruction Objective as a function of $\beta$. (D) Conditional Independence Loss as a function of $\beta$.  }
	\label{fig:panel2}
\end{figure}

\section{Discussion and Conclusion}\label{discussion}

In this paper, we examined the learning of disentangled representations by extracting statistically independent latent variables in $\beta$-VAE. We proved general theorems on variational Bayesian inference in the context of $\beta$-VAE and introduced an analytically tractable $\beta$-VAE model. We also performed experiments on synthetic datasets to test our insights from the general theorems and the tractable model, and found good agreements.

$\beta$-VAE enforces conditional independence of its units at the bottleneck layer. This preference is not compatible with independence of latent variables, and therefore may lead to an optimal value of $\beta$ for latent variable inference.


There are other perspectives on what constitutes a disentangled representation not addressed in this paper\cite{bengio2013representation,burgess2018understanding}, including definitions not statistical in nature, instead taking into account the manifold structure and symmetry transformations in data  \cite{bengio2013representation,dicarlo2007untangling,higgins2018towards}. Other deep learning approaches to disentangling include the adversarial setting \cite{denton2017unsupervised, tran2017disentangled,john2018disentangled}. Disentangled representations have also been studied in supervised and semi-supervised contexts \cite{siddharth2017learning}.

\appendices
	\section{Proof of Proposition \ref{prop2}}\label{pprop2}
	
	We prove a more general version of eq. \eqref{eq:qp} given in Prop. \ref{prop2}. Eq. \eqref{eq:re} follows from eq. \eqref{eq:qp} and Prop. \ref{prop1}.
	
	\begin{proposition}
		Consider an objective function given by a sum of two terms,
		\begin{align}\label{Odef}
		O(\bkappa;\beta) = A(\bkappa) - \beta B(\bkappa),
		\end{align}
		to be maximized over parameters $\bkappa$, and $\beta$ is a hyperparameter. Let $\bkappa^*(\beta) = \underset{\bkappa}{\argmax} \, O(\bkappa,\beta)$. As $\beta$ increases $B(\kappa^*(\beta))$ is nonincreasing.
	\end{proposition}
	
	\begin{proof}
		The proof uses contradiction. Let $\beta_2 > \beta_1$ and 
		\begin{align} \label{kappa}
		\bkappa_1 \equiv \bkappa^*(\beta_1),\qquad  \bkappa_2 \equiv \bkappa^*(\beta_2).
		\end{align}
		Then 
		\begin{align}\label{Ok}
		O(\bkappa_1,\beta_1) &= O(\bkappa_1,\beta_2) + (\beta_2-\beta_1)B(\bkappa_1) \nonumber \\ 
		& \leq O(\bkappa_2,\beta_2) + (\beta_2-\beta_1)B(\bkappa_1),
		\end{align}
		where the first line is an identity, and the second line follows from the optimality of $\bkappa_2$ at $\beta=\beta_2$.
		
		Now we assume $B(\bkappa_2) > B(\bkappa_1)$, and see that this leads to a contradiction.
		\begin{align}
		&O(\bkappa_2,\beta_2) + (\beta_2-\beta_1)B(\bkappa_1)  \nonumber \\
		&\qquad <  O(\bkappa_2,\beta_2) + (\beta_2-\beta_1)B(\bkappa_2)  = O(\bkappa_2,\beta_1).
		\end{align}     
		The inequality follows from our assumption, and the equality from \eqref{Odef}. Combined with \eqref{Ok}, this implies
		\begin{align}
		O(\bkappa_1,\beta_1) < O(\bkappa_2,\beta_1)
		\end{align}
		which contradicts \eqref{kappa}. Therefore if $\beta_2>\beta_1$, then $B(\kappa_2) \leq  B(\kappa_1)$.
	\end{proof}

	\section{Details of the analytically tractable $\beta$-VAE model}

	\subsection{Integrating out data from the objective}
	\label{app:intoutdata}
	The full $\beta$-VAE objective is \eqref{bVAEobj} averaged with respect to the data distribution $p(\x)$:
	\begin{align}
	& L(\bm{\theta},\bm{\phi};\beta) \equiv \mathbb{E}_{p(\x)}\left[ \mathcal{L}(\bm{\theta},\bm{\phi};\beta)\right]  \nonumber \\
	&=\mathbb{E}_{p(\x)} \left[\mathbb{E}_{q_{\bphi}(\z | \x)}\left[\log p_{\btheta}(\x | \z)\right]-\beta D_{K L}\left(q_{\bphi}(\z | \x) \| p(\z)\right) \right].
	\end{align}
	We first calculate $\mathbb{E}_{q_{\bphi}(\z | \x)}\left[\log p_{\btheta}(\x | \z)\right]$. We use the reparametrization trick: For $\z  \sim q_{\bphi}(\z | \x) = \N(\bmu_{\z},\bSigma_{\z})$, we can write $\z = \bmu_{\z} + \bSigma_{\z}^{1/2} \bm{\epsilon}$ with $\bm{\epsilon}\sim \N(\mathbf{0},\mathbf{I}_k)$. Then,
	\begin{align*}
	&\mathbb{E}_{\z \sim q_{\bphi}(\mathbf{z} | \mathbf{x})} [ \log p_{\btheta}(\mathbf{x} | \mathbf{z})] \\
	&\quad=\mathbb{E}_{\z \sim q_{\bphi}(\mathbf{z} | \mathbf{x})} \bigg[ \log \mathcal{N}(\mathbf{x}; \mathbf{D}\z + \mathbf{b}^{D}, \mathbf{I}_N) \bigg] \\ 
	&\quad= \mathbb{E}_{\bm{\epsilon} \sim \mathcal{N}(0,1)} \bigg[\log \mathcal{N}(\mathbf{x}; \mathbf{D}(\bmu_{\z} +  \bSigma_{\z}^{1/2} \bm{\epsilon}) + \mathbf{b}^{D},  \mathbf{I}_N ) \bigg] \\
	&\quad= -\frac{N}{2}\log(2\pi) - \frac{1}{2}(\D \bmu_{\z} + \mathbf{b}^{D} - \mathbf{x})^2 \\
	&\qquad\quad- \frac{1}{2} \mathbb{E}_{\bm{\epsilon} \sim \mathcal{N}(0,1)}\bigg[\bm{\epsilon}^{\top} (\D \bSigma_{\z}^{1/2})^{\top} (\D \bSigma_{\z}^{1/2})\bm{\epsilon} \bigg].
	\end{align*}
	The last term can be calculated by the following useful trick. Let's introduce a source term $\J$ into the generating functional,
	\begin{align}
	Z[\J] = \int \frac{d\z}{(2\pi)^{n/2} \sqrt{\det \bSigma_{\z}}} \exp \bigg(-\frac{1}{2}\z^{\top} \bSigma_{\z}^{-1}\z + \J^{\top} \A \z \bigg),
	\end{align}
	then differentiating with respect to the source, 
	\begin{align}
	\bigg(\frac{\delta}{\delta \J} \bigg)^{\top} \bigg(\frac{\delta}{\delta \J} \bigg) Z[\J] \bigg\vert_{\J=0} =     \mathbb{E}_{\z \sim \mathcal{N}(0,\bSigma_{\z})}\left[\z^{\top} \A^{\top} \A \z \right] 
	\end{align}
	On the other hand, we can perform the Gaussian integral in $Z[\J]$ to obtain,
	\begin{align}
	Z[\J] = \exp \bigg\{\frac{1}{2}(\J\A)^{\top} \bSigma_{\z} (\J\A)  \bigg\}.
	\end{align}
	Then we arrive at 
	\begin{align}
	\label{eqn:correlation}
	\mathbb{E}_{\z \sim \mathcal{N}(0,\bSigma_{\z})}\left[\z^{\top} \A^{\top} \A \z \right] = \text{Tr} (\A \bSigma_{\z} \A^{\top})
	\end{align}
	Eq. \eqref{eqn:correlation} is central to the calculations of many results presented in the text.
	
	Going back to the reconstruction objective, using eq. \eqref{eqn:correlation} we have (up to constants)
	\begin{align}
	\mathbb{E}_{\z \sim q_{\bphi}(\mathbf{z} | \mathbf{x})} [ \log p_{\btheta}(\mathbf{x} | \mathbf{z})] &= - (\D \bmu_{\z} + \mathbf{b}^{D} - \mathbf{x})^2 \nonumber \\
	&\qquad\qquad- \text{Tr}(\D^{\top}\D \bSigma_{\z}).
	\end{align}{}
	Similarly we can calculate the conditional independence loss,
	\begin{align}
	&D_{K L}(q_{\bphi}(\mathbf{z} | \mathbf{x}) \| p(\mathbf{z})) \nonumber \\
	&\qquad\qquad = -\frac{1}{2}\left(k + \text{Tr}\log \bSigma_{\z} - \bmu_{\z}^{\top} \bmu_{\z} - \text{Tr} \bSigma_{\z} \right).
	\end{align}
	Putting everything together, the objective function we want to maximize is (neglecting constant terms)
	\begin{align}
	&L(\bm{\theta},\bm{\phi};\beta) =  \frac{1}{2}\mathbb{E}_{p(\mathbf{x})}\left[ - (\D \bmu_{\z} + \mathbf{b}^{D} - \mathbf{x})^{\top}(\D \bmu_{\z} + \mathbf{b}^{D} - \mathbf{x})\right. \nonumber \\
	&\quad - \beta \bmu_{\z}^{\top}\bmu_{\z} - \text{Tr}(\D^{\top}\D \bSigma_{\z}) + \beta \text{Tr}\log \bSigma_{\z} - \beta\text{Tr}\bSigma_{\z} ].
	\end{align}
	The expectation with respect to $\x$ amounts to performing Gaussian integrals in $\x$, as $\x \sim \N(\mathbf{0},\bSigma_{\x})$, and thus can be done exactly. After plugging in the definition of $\bmu_{\z}, \bSigma_{\z}$ from eq. \eqref{eqn:musigma}, and performing the $\x$ integrals, the result is given in eq. \eqref{eqn:objective_nox}.

	\subsection{Taking derivatives of the objective}
	\label{app:derivatives}
	In order to take derivatives of eq. \eqref{eqn:objective_nox}, we unpack the indices (to ease the notation, we denote $\bSigma_{\x}$ as $\bSigma$, and follow the Einstein summation convention, repeated indices are to be summed over unless the summation is explicitly specified)
	\begin{align}
	L = &-\frac{1}{2} \bigg\{ (D_{ij} W^{\mu}_{jk} - \delta_{ik}) \Sigma_{kl} (W^{\mu}_{ml} D_{im} - \delta_{il}) - \beta \sum_i b^{\sigma}_i \nonumber \\
	&+ \beta (W^{\mu})_{ij} \Sigma_{jk} W^{\mu}_{ik} +(D_{ij} b^{\mu}_j + b^D_i)^2 + \beta (b^{\mu}_i)^2 \nonumber \\
	&+ \sum_i \bigg( D^2_{li} + \beta \bigg) \exp \bigg(\frac{1}{2} W^{\sigma}_{ij} \Sigma_{jk} W^{\sigma}_{ik} + b_i^{\sigma} \bigg)  \bigg\}
	\end{align}
	Then,
	\begin{align}
	0 &= \frac{\partial L}{\partial W^{\mu}_{ab}} = \left [ \left( \D^{\top} (\D \W^{\mu} - \mathbf{I}_N) + \beta \W^{\mu} \right) \bSigma \right]_{ab}, \\
	0 &= \frac{\partial L}{\partial b^{\mu}_{a}} = \left[ (\D \bb^{\mu} + \bb^D) \D + \beta \bb^{\mu} \right]_a, \\
	0 & = \frac{\partial L}{\partial D_{ab}} = \left[ (\D \W^{\mu} - \I_N) \bSigma (\W^{\mu})^{\top} \right]_{ab} \nonumber \\
	& \qquad \qquad  + \left[\D \bb^{\mu} + \bb^D\right]_a b^{\mu}_b + D_{ab} e^{ \frac{1}{2}\left[\W^{\sigma}\bSigma (\W^{\sigma})^{\top}\right]_{bb} + b^{\sigma}_b }, \\
	0 &= \frac{\partial L}{\partial b^{D}_{a}} = \left[\D \bb^{\mu} + \bb^D\right]_a \\
	0 &= \frac{\partial L}{\partial W^{\sigma}_{ab}} = \left(\left[\D^{\top}\D\right]_{aa} + \beta \right) e^{ \frac{1}{2}\left[\W^{\sigma}\bSigma\right]_{ab}W^{\sigma}_{ab} + b^{\sigma}_a } \left[\W^{\sigma}\bSigma\right]_{ab}, \\
	0 &= \frac{\partial L}{\partial b^{\sigma}_{a}} = \left( \left[\D^{\top}\D\right]_{aa} + \beta \right) e^ { \frac{1}{2}\left[\W^{\sigma}\C (\W^{\sigma})^{\top}\right]_{aa} + b^{\sigma}_a } - \beta.
	\end{align}
	From the $b^{\mu}_a$ and $b^D_a$ equations we can immediately see $\bb^{\mu} = \bb^D = 0$.  
	
	\subsection{Derivation of the ground-truth posterior}
	\label{app:gtposterior}
	We observe that since both $\s$ and $\bm{\eta}$ are independently normally distributed in \eqref{source}, $\s$ and $\bm{\eta}$ are jointly normal, i.e., $p(\s,\bm{\eta})$ is a normal distribution. However, note that $p(\s,\bm{\eta})$ is just $p(\s,\x)$ up to a coordinate transformation, so $p(\s,\x)$ is also normal. Also, as $\s \in \mathbb{R}^{k}$, $\x \in \mathbb{R}^N$, $(\s, \x) \in \mathbb{R}^{N+k}$. We can think of $\s$ and $\x$ partition a $(N+k)$-dimensional normal distribution $p((\s,\x))$. Therefore, to find the conditional probability $p_{\text{g-t}}(\s |\x)$, we can just use the formula for conditioning multivariate normal distribution: 
	\begin{equation}
	p_{\text{g-t}}(\s |\x) = \mathcal{N}(\bmu_{\s|\x}, \bSigma_{\s|\x}), 
	\end{equation}
	where 
	\begin{equation}
	\begin{split}
	\bmu_{\s|\x} &= \bmu_{\s} + \text{Cov}(\x,\s)^{\top}(\bSigma_{\x})^{-1}(\x - \bmu_{\x}) \\
	\bSigma_{\s|\x} &= \bSigma_{\s} - \text{Cov}(\x,\s)^{\top}(\bSigma_{\x})^{-1}\text{Cov}(\x,\s).
	\end{split}{}
	\end{equation}{}
	Now specializing to our case \eqref{source}, 
	\begin{equation}
	\text{Cov}(\x,\s) = \text{Cov}(\A \s + \bm{\eta}, \s)= \A \bSigma_{\s} + \text{Cov}(\s,\bm{\eta})= \A.
	\end{equation}
	Note that $\bSigma_{\x} = \A \A^{\top} + \mathbf{I}_N$, then 
	\begin{equation}
	\bmu_{\s|\x} = \A^{\top}(\A \A^{\top} + \mathbf{I}_N)^{-1}\x = (\A^{\top}\A + \mathbf{I}_k)^{-1}\A^{\top}\x,
	\end{equation}
	where in the second equality we have used the \textit{matrix push-through identity}:
	For any matrices $\mathbf{U} \in \mathbb{R}^{N\times k}$,$\mathbf{V} \in \mathbb{R}^{k \times N}$,
	\begin{equation}
	\label{eqn:push}
	(\mathbf{I}_N + \mathbf{U}\mathbf{V})^{-1}\mathbf{U} = \mathbf{U}(\mathbf{I}_k + \mathbf{V}\mathbf{U})^{-1}.
	\end{equation}
	Now the covariance,
	\begin{equation}
	\begin{aligned}
	\bSigma_{\s|\x} &= \mathbf{I}_k - \A^{\top}(\A \A^{\top} + \mathbf{I}_N)^{-1} \A \\
	&= \mathbf{I}_k - \A^{\top}\A (\A^{\top} \A + \mathbf{I}_k)^{-1} \\
	&= \mathbf{I}_k - \A^{\top}\A [(\A^{\top}\A)^{-1} - (\A^{\top}\A )^{-1}(\A^{\top}\A + \mathbf{I}_k)^{-1}] \\
	&= (\A^{\top}\A + \mathbf{I}_k)^{-1},
	\end{aligned}{}
	\end{equation}
	where in the third equality we have used the \textit{Woodbury matrix identity}: For any invertible matrix $\B \in \mathbb{R}^{N \times N}$ and size compatible matrices $\mathbf{U} \in \mathbb{R}^{N\times k}$ and $\mathbf{V} \in \mathbb{R}^{k \times N}$:
	\begin{equation}
	\label{eqn:woodbury}
	(\B+\U\V)^{-1} = \B^{-1} - \B^{-1}\U(\mathbf{I}_k+\V\B^{-1}\U)^{-1}\V\B^{-1}.
	\end{equation}{}

	\subsection{Derivation of the model posterior}
	\label{app:modelposterior}
	Our goal is to use the Bayes rule to calculate the model posterior, $p_{\btheta}(\x|\z) = p_{\btheta}(\x|\z)p(\z)/p_{\btheta}(\x).$
	
	In order to do so, we first need to calculate the \textit{evidence} $p_{\btheta}(\x)$,
	\begin{align}
	p_{\btheta}(\x) &= \int_{\mathbb{R}^{k}} d\z p_{\btheta}(\x|\z) p(\z) \\
	&= \int_{\mathbb{R}^{k}} d\z \N (\D \z, \mathbf{I}_N)\N(\mathbf{0},\mathbf{I}_k) \\
	&= \N(\mathbf{0},(\D\D^{\top}+\mathbf{I}_N)),
	\end{align}
	where in the third equality we have used eq.s \eqref{eqn:push} and \eqref{eqn:woodbury} to simplify. 
	Therefore, 
	\begin{equation}
	p_{\btheta}(\x|\z) = \frac{\N (\D \z, \mathbf{I}_N)\N(\mathbf{0},\mathbf{I}_k)}{\N(\mathbf{0},(\D\D^{\top}+\mathbf{I}_N))}
	\end{equation}
	After some simplifications  using eq.s \eqref{eqn:push} and \eqref{eqn:woodbury}, we arrived at
	\begin{align}
	p_{\btheta}(\x|\z) &= \N((\D\D^{\top}+\mathbf{I}_N)^{-1}\D^{\top}, (\D\D^{\top}+\mathbf{I}_N)^{-1}) \nonumber \\
	&\equiv \mathcal{N}(\bmu_{\z|\x}, \bSigma_{\z|\x}).
	\end{align}
	\subsection{Derivation of $\rm{MIE/TIE}$}
	First let's consider $\rm{MIE}$. Let
	\begin{align}
	\bmu_{\z|\x} \equiv \F \x, \qquad 
	\bSigma_{\z|\x} \equiv \E.
	\end{align}{}
	Then, we can write $\rm{MIE}$ as 
	\begin{align}
	\rm{MIE} =& \mathbb{E}_{p(\x)}\big[ D_{K L}(q_{\bphi}(\mathbf{z} | \mathbf{x}) \| p_{\btheta}(\x|\z)) \big] \nonumber \\
	= &\frac{1}{2}\mathbb{E}_{p(\x)}\bigg[(\F\x - \bmu_{\z})^{\top}\E^{-1}(\F\x - \bmu_{\z}) \nonumber \\
	& + \text{Tr}(\E^{-1}\bSigma_{\z}) - \log \bigg(\frac{\det\bSigma_{\z}}{\det\E} \bigg) - k \bigg].
	\end{align}
	Plugging in eq. \eqref{eqn:musigma} and performing the $\x$ Gaussian integrals as in Appendix \ref{app:intoutdata}, we arrive at eq. \eqref{eqn:mie/tie}.
	
	Note that at network optimum, our model posterior $p_{\btheta}(\z|\x)$ equals to the ground-truth posterior $p(\s|\x)$ upon changing $\D$ to $\A$. Therefore, we just need to replace $\D$ by $\A$ in the above derivation to obtain the results for $\rm{TIE}$.

\section{Simulation Details}\label{app_sims}
The deep neural network models used for the numerical experiments  task used the same overall architecture. The encoder is a feed forward network with 3 hidden layers, with 256, 200, and 200 units. 2 parallel hidden layers with 2 neurons parameters the mean and variance for $k=2$ latent variables. The decoder consists of 3 feed-forward hidden layers with 200, 200, and 256 units, then outputs the reconstructed image. 
The network was trained for 1000 epochs over the entire synthetic dataset, comprising of 1000 examples. We used a tanh activation function used along with Adam Optimization \cite{kingma2014adam} with a learning rate of 1e-3. Experiments were repeated across 300 realizations for each $\beta$ value. Results shown were averaged over the whole set of realizations. 

The Reconstruction Objective was calculated for each trained model through generating 1000 samples from the encoder, passing them to the decoder to approximately calculate  $\mathbb{E}_{q_{\bphi}(\z | \x)}\left[\log p_{\btheta}(\x | \z)\right]$,
and averaging over the data $\x$. The Conditional Independence Loss was calculated directly using the Tensorflow Distributions library's native KL Divergence method. The ELBO was calculated by numerically taking the difference of these two terms, and the $\beta$-VAE objective was an extension of this with the hyperparameter $\beta$ included. The Inference Error was calculated numerically using the modelled $\bmu_{\z}$ and $\bSigma_{\z}$ and estimating $p(\x)$ from mini-batches.

In Fig. \ref{fig:panel3}, we show results on another simulation consistent with our findings.


\begin{figure}
	\centering
	\includegraphics[width=9cm]{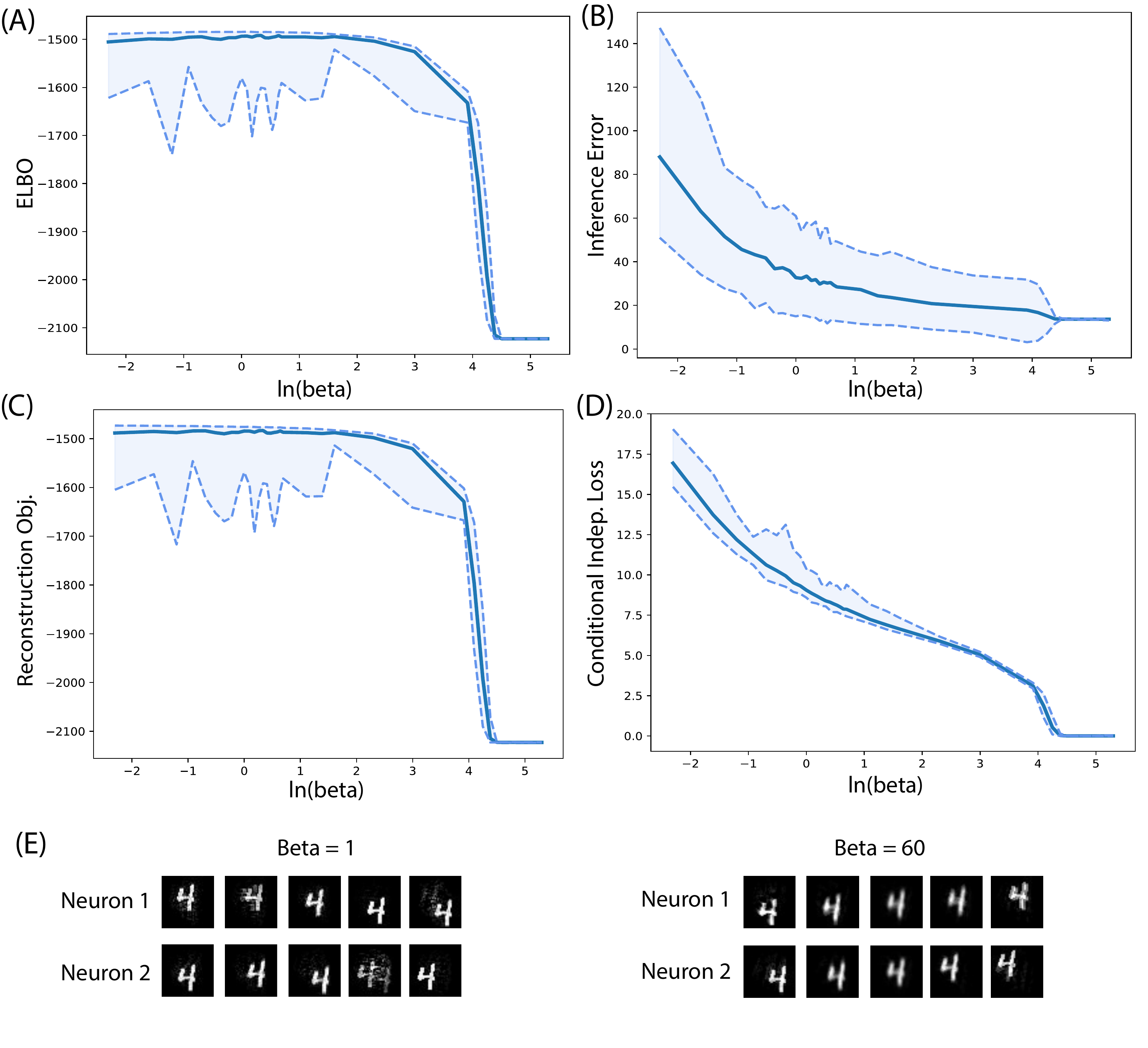}
	\caption{Values for error terms across 300 random initializations of the network for a synthetic dataset, which comprises of a single MNIST digit localized at different locations on a blank canvas. The cartesian coordinate of the digit in a sample from our data, $\x$, is determined by eq. \eqref{source}, with $A_{ij} = 2\delta_{ij} + 0.73$, $\s \sim \N(\bm{0},\I_k)$, $ \bm{\eta} \sim \N(\bm{0},\I_N)$,  $N=k=2$.  Dashed lines represent the minimum and maximum values, and solid line represents the average. (A) ELBO as a function of $\beta$. (B) TIE as a function of $\beta$. Its minima over various random initialization follow a non-monotonic trend. (C) Reconstruction objective as a function of $\beta$. (D) Conditional Independence Loss as a function of $\beta$.  (E) Traversal of latent encoding in bottleneck neurons for small and large $\beta$. One neuron is held fixed while the other is modulated to generate reconstructions. Reconstruction of the digits noticeable worsens with higher $\beta$, while units in the bottleneck encode for structured, orthogonal axes of motion.}
	\label{fig:panel3}
\end{figure}


\begin{thebibliography}{10}
	\providecommand{\url}[1]{#1}
	\csname url@samestyle\endcsname
	\providecommand{\newblock}{\relax}
	\providecommand{\bibinfo}[2]{#2}
	\providecommand{\BIBentrySTDinterwordspacing}{\spaceskip=0pt\relax}
	\providecommand{\BIBentryALTinterwordstretchfactor}{4}
	\providecommand{\BIBentryALTinterwordspacing}{\spaceskip=\fontdimen2\font plus
		\BIBentryALTinterwordstretchfactor\fontdimen3\font minus
		\fontdimen4\font\relax}
	\providecommand{\BIBforeignlanguage}[2]{{%
			\expandafter\ifx\csname l@#1\endcsname\relax
			\typeout{** WARNING: IEEEtran.bst: No hyphenation pattern has been}%
			\typeout{** loaded for the language `#1'. Using the pattern for}%
			\typeout{** the default language instead.}%
			\else
			\language=\csname l@#1\endcsname
			\fi
			#2}}
	\providecommand{\BIBdecl}{\relax}
	\BIBdecl
	
	\bibitem{bengio2013representation}
	Y.~Bengio, A.~Courville, and P.~Vincent, ``Representation learning: A review
	and new perspectives,'' \emph{IEEE transactions on pattern analysis and
		machine intelligence}, vol.~35, no.~8, pp. 1798--1828, 2013.
	
	\bibitem{olshausen1997sparse}
	B.~A. Olshausen and D.~J. Field, ``Sparse coding with an overcomplete basis
	set: A strategy employed by v1?'' \emph{Vision research}, vol.~37, no.~23,
	pp. 3311--3325, 1997.
	
	\bibitem{hinton1994autoencoders}
	G.~E. Hinton and R.~S. Zemel, ``Autoencoders, minimum description length and
	helmholtz free energy,'' in \emph{Advances in neural information processing
		systems}, 1994, pp. 3--10.
	
	\bibitem{alemi2018information}
	A.~Alemi, B.~Poole, I.~Fischer, J.~Dillon, R.~A. Saurus, and K.~Murphy, ``An
	information-theoretic analysis of deep latent-variable models,'' 2018.
	
	\bibitem{kingma2013auto}
	D.~P. Kingma and M.~Welling, ``Auto-encoding variational bayes,'' \emph{arXiv
		preprint arXiv:1312.6114}, 2013.
	
	\bibitem{higgins2017beta}
	I.~Higgins, L.~Matthey, A.~Pal, C.~Burgess, X.~Glorot, M.~Botvinick,
	S.~Mohamed, and A.~Lerchner, ``beta-vae: Learning basic visual concepts with
	a constrained variational framework.'' \emph{ICLR}, vol.~2, no.~5, p.~6,
	2017.
	
	\bibitem{kingma2019introduction}
	D.~P. Kingma and M.~Welling, ``An introduction to variational autoencoders,''
	\emph{arXiv preprint arXiv:1906.02691}, 2019.
	
	\bibitem{hyvarinen2000independent}
	A.~Hyv{\"a}rinen and E.~Oja, ``Independent component analysis: algorithms and
	applications,'' \emph{Neural networks}, vol.~13, no. 4-5, pp. 411--430, 2000.
	
	\bibitem{khemakhem2019variational}
	I.~Khemakhem, D.~P. Kingma, and A.~Hyv{\"a}rinen, ``Variational autoencoders
	and nonlinear ica: A unifying framework,'' \emph{arXiv preprint
		arXiv:1907.04809}, 2019.
	
	\bibitem{tipping1999probabilistic}
	M.~E. Tipping and C.~M. Bishop, ``Probabilistic principal component analysis,''
	\emph{Journal of the Royal Statistical Society: Series B (Statistical
		Methodology)}, vol.~61, no.~3, pp. 611--622, 1999.
	
	\bibitem{locatello2018challenging}
	F.~Locatello, S.~Bauer, M.~Lucic, S.~Gelly, B.~Sch{\"o}lkopf, and O.~Bachem,
	``Challenging common assumptions in the unsupervised learning of disentangled
	representations,'' \emph{arXiv preprint arXiv:1811.12359}, 2018.
	
	\bibitem{higgins2018towards}
	I.~Higgins, D.~Amos, D.~Pfau, S.~Racaniere, L.~Matthey, D.~Rezende, and
	A.~Lerchner, ``Towards a definition of disentangled representations,''
	\emph{arXiv preprint arXiv:1812.02230}, 2018.
	
	\bibitem{dayan2001theoretical}
	P.~Dayan, L.~F. Abbott \emph{et~al.}, \emph{Theoretical neuroscience}.\hskip
	1em plus 0.5em minus 0.4em\relax Cambridge, MA: MIT Press, 2001, vol. 806.
	
	\bibitem{jordan1999introduction}
	M.~I. Jordan, Z.~Ghahramani, T.~S. Jaakkola, and L.~K. Saul, ``An introduction
	to variational methods for graphical models,'' \emph{Machine learning},
	vol.~37, no.~2, pp. 183--233, 1999.
	
	\bibitem{doersch2016tutorial}
	C.~Doersch, ``Tutorial on variational autoencoders,'' \emph{arXiv preprint
		arXiv:1606.05908}, 2016.
	
	\bibitem{burgess2018understanding}
	C.~P. Burgess, I.~Higgins, A.~Pal, L.~Matthey, N.~Watters, G.~Desjardins, and
	A.~Lerchner, ``Understanding disentangling in $\beta$-vae,'' \emph{arXiv
		preprint arXiv:1804.03599}, 2018.
	
	\bibitem{Huang_2018_ECCV}
	X.~Huang, M.-Y. Liu, S.~Belongie, and J.~Kautz, ``Multimodal unsupervised
	image-to-image translation,'' in \emph{The European Conference on Computer
		Vision (ECCV)}, September 2018.
	
	\bibitem{lample2017fader}
	G.~Lample, N.~Zeghidour, N.~Usunier, A.~Bordes, L.~Denoyer, and M.~Ranzato,
	``Fader networks: Manipulating images by sliding attributes,'' in
	\emph{Advances in Neural Information Processing Systems}, 2017, pp.
	5967--5976.
	
	\bibitem{karras2019style}
	T.~Karras, S.~Laine, and T.~Aila, ``A style-based generator architecture for
	generative adversarial networks,'' in \emph{Proceedings of the IEEE
		Conference on Computer Vision and Pattern Recognition}, 2019, pp. 4401--4410.
	
	\bibitem{dicarlo2007untangling}
	J.~J. DiCarlo and D.~D. Cox, ``Untangling invariant object recognition,''
	\emph{Trends in cognitive sciences}, vol.~11, no.~8, pp. 333--341, 2007.
	
	\bibitem{denton2017unsupervised}
	E.~L. Denton \emph{et~al.}, ``Unsupervised learning of disentangled
	representations from video,'' in \emph{Advances in neural information
		processing systems}, 2017, pp. 4414--4423.
	
	\bibitem{tran2017disentangled}
	L.~Tran, X.~Yin, and X.~Liu, ``Disentangled representation learning gan for
	pose-invariant face recognition,'' in \emph{Proceedings of the IEEE
		Conference on Computer Vision and Pattern Recognition}, 2017, pp. 1415--1424.
	
	\bibitem{john2018disentangled}
	V.~John, L.~Mou, H.~Bahuleyan, and O.~Vechtomova, ``Disentangled representation
	learning for text style transfer,'' \emph{arXiv preprint arXiv:1808.04339},
	2018.
	
	\bibitem{siddharth2017learning}
	N.~Siddharth, B.~Paige, J.-W. Van~de Meent, A.~Desmaison, N.~Goodman, P.~Kohli,
	F.~Wood, and P.~Torr, ``Learning disentangled representations with
	semi-supervised deep generative models,'' in \emph{Advances in Neural
		Information Processing Systems}, 2017, pp. 5925--5935.
	
	\bibitem{kingma2014adam}
	D.~P. Kingma and J.~Ba, ``Adam: A method for stochastic optimization,''
	\emph{arXiv preprint arXiv:1412.6980}, 2014.
	
\end{thebibliography}
\end{document}